%% file: TechReport.tex
\documentclass[hidelinks,12pt,a4paper]{article}
\usepackage{times}

\usepackage[numbers]{natbib}
\usepackage{multicol}
\usepackage[bookmarks=true]{hyperref}

\usepackage{amsmath}
\usepackage{rotating}
\usepackage{epsfig} 
\usepackage{subfigure}
\usepackage{txfonts}
\usepackage{floatflt, lscape, graphicx}
\usepackage{wrapfig}
  \graphicspath{{figs/}}
  \DeclareGraphicsExtensions{.pdf,.jpg,.png,.eps}  

\usepackage{latexsym}
\usepackage[ruled]{algorithm}
\usepackage[noend]{algorithmic}
\usepackage{xspace}
\usepackage{longtable}
\usepackage{blindtext}
\usepackage{xcolor}
\usepackage{multirow,multicol}
\usepackage{rotating}
\usepackage[font=normalsize,labelfont=bf]{caption}
\usepackage{tcolorbox}

\usepackage[absolute,overlay]{textpos}

\include{shortcut-forArxiv}

\begin{document}


\title{\LARGE \bf An NCAP-like Safety Indicator for Self-Driving Cars}
\date{}

\author{Jimy Cai Huang \hspace{2cm} Hanna Kurniawati \\
School of Computing (fka. Research School of Computer Science) \\
Australian National University\\
Email: \{jimy.cai, hanna.kurniawati\}@anu.edu.au
}



%

\maketitle

\begin{abstract}
This paper proposes a mechanism to assess the safety of autonomous cars. 
It assesses the car's safety in scenarios where the car must avoid collision with an adversary.  Core to this mechanism is a safety measure, called  \noaLong (\noa), which computes the average similarity between sets of safe adversary's trajectories and kamikaze trajectories close to the safe trajectories. The kamikaze trajectories are generated based on planning under uncertainty techniques, namely the Partially Observable Markov Decision Processes, to account for the partially observed car policy from the point of view of the adversary. We found that \noa is inversely proportional to the upper bound on the probability that a small deformation changes a collision-free trajectory of the adversary into  a colliding one. We perform systematic tests on a scenario where the adversary is a pedestrian crossing a single-lane road in front of the car being assessed ---which is, one of the scenarios in the Euro-NCAP's  Vulnerable Road User (VRU) tests on Autonomous Emergency Braking. Simulation results on assessing cars with basic controllers and a test on a Machine-Learning controller using a high-fidelity simulator indicates promising results for \noa to measure the safety of autonomous cars. Moreover, the time taken for each simulation test is under 11 seconds, enabling a sufficient statistics to compute \noa from simulation to be generated on a quad-core desktop in less than 25 minutes.
\end{abstract}


\section{INTRODUCTION}

Safety of robotics and autonomous systems, specifically autonomous cars, have become increasingly important. Throughout this paper, the term \emph{autonomous} includes semi-autonomous systems too.  Many work have been proposed to improve the safety of autonomous cars. Most focus on developing autonomous car components (e.g., control and machine learning) with safety guarantees (e.g., \cite{hawkins2021guidance, schwalbe:hal-02442819, shalev2016safe,  Wongpiromsarn:2011:Synthesis}). This paper aims to explore an orthogonal issue, namely the safety assessment. 

Recent work have started to focus on developing testing mechanisms to assess the safety of autonomous cars. Most work in this direction focus on identifying critical testing scenarios\cite{capito2020modeled, Kel18:Scalable}. Since accidents have a small percentage, identifying assessment scenarios that lead to accidents, especially catastrophic accidents, is difficult. Rare event simulations\cite{Kel18:Scalable} and a variety of adversary generation strategies\cite{capito2020modeled, sun2021corner, Wen20:Model} have been proposed to identify such scenarios. Our proposed safety assessment mechanism can benefit from these work too. 

However, in this paper, the purpose of our testing mechanism is not to identify problematic scenarios per se. Rather, we aim to develop a testing mechanism that can eventually help users to easily compare the safety of different autonomous cars, including different versions of the software that run them. Such a safety indicator is akin to the New Car Assessment Program (NCAP) safety rating that has helped users with non-autonomous cars. However, since NCAP testing scenarios are mostly static and performed at most once in the lifetime of a car, we do need to adjust the testing mechanism and safety indicator to be adaptive and efficient enough, such that they are  suitable for autonomous systems and frequent assessment is viable. 
Our proposed testing mechanism is based on the observation that safe autonomous cars must provide sufficient room for errors and uncertainty, in particular due to non-deterministic effects of actions. For instance, different drivers' reaction time in taking over  control to avoid colliding with a pedestrian may result in different outcomes, different road and tyres conditions may result in unexpected collision with a pedestrian, etc.. Therefore, we propose to measure the safety of an autonomous car based on how likely will a safe scenario change into a dangerous one under a small perturbation of the scenario. 


To that end, our mechanism will assess the car's safety in scenarios where it must avoid collision with an adversary. Core to our mechanism is a safety measure, called \noaLong (\noa), which is based on the average similarity measure between safe and kamikaze trajectories of the adversary when interacting with the car being assessed. A kamikaze trajectory is an adversary's trajectory that causes the adversary to collide with the car being tested. Given a safe trajectory for the adversary, our mechanism computes multiple kamikaze trajectories closest to the safe trajectory. The safety measure \noa is then the average distance between samples of pairs of safe and kamikaze trajectories. We show the probability that a small deformation changes a collision-free trajectory of the adversary  into a colliding one is upper bounded by a value inversely proportional to \noa. 
 
Our assessment mechanism is general enough to be used with a variety of adversaries, but in this paper, our experiments focus on scenarios where the adversary is a pedestrian crossing a single-lane road in front of the assessed car ---one of the scenarios in the Euro-NCAP's  Vulnerable Road User (VRU) tests on Autonomous Emergency Braking\cite{EuroNcap}. Results on basic controllers and a Machine Learning controller as provided by the Carla\cite{CARLA17} simulator indicate that \noa increases as collision rate decreases. Moreover, our results indicate the time required to compute \noa makes the assessment mechanism to be potentially viable to be performed frequently, such as after every software updates.

\section{Related Work and Background}

\subsection{Related Work}

To ensure safety of autonomous cars, many work have focused on using formal methods for verification of the autonomous vehicle system (e.g., \cite{althoff2014online, Wongpiromsarn:2011:Synthesis}). A short summary of formal methods for autonomous cars is provided in\cite{seshia2015formal}. These approaches require formal specifications, which is often not easy to construct completely due to the complexity of the system and the sheer possibilities of the different scenarios that an autonomous car may encounter.

Another line of work is to develop testing mechanisms to assess the safety of autonomous cars. Recently, work in this direction have focused on generating test scenarios that will lead to accidents\cite{capito2020modeled, Kel18:Scalable, sun2021corner, Wen20:Model}. This problem is difficult because accidents are relatively rare. 


Obviously, a mechanism to test the safety of a car is not new. The well-accepted NCAP testing protocol\cite{GlobalNcap} was introduced in 1979. However, testing scenarios developed by NCAP are mostly static, which is not suitable for autonomous cars. In this paper, we propose a testing mechanism that can utilise these static scenarios as safe trajectories, and then find a kamikaze trajectory close to these safe trajectories to compute \noa. We hope such a mechanism would be more acceptable for regulatory purposes.

To generate safe and kamikaze trajectories, the testing mechanism needs to have a predictive model of the car's behaviour, albeit imperfect. For this purpose, many work can be adopted, such as \cite{Bha20:Online, Fri20:Efficient,  Hoe17:Probabilistic, Sch18:Multiple}. Our kamikaze trajectory generation can also benefit from work on pursuit evasion\cite{chung2011search} and more recently \cite{Kor18:Adaptive}, though the latter is focused on deforming observations rather than an adversary's behaviour.

\subsection{Background}

Two main components in our proposed mechanism is the safety measure \noa, which relies on Fr\'{e}chet distance, and the kamikaze trajectory generation, which is based on the Partially Observable Markov Decision Processes (POMDPs). The following are backgrounds on these two concepts.

\subsubsection{Fr\'{e}chet Distance}

Fr\'{e}chet distance\cite{har2011geometric} measures the similarity between two curves while adhering to the order of points on the curve. Suppose $P: [0,1] \rightarrow {\mathbb{R}}^n$ and $Q: [0, 1] \rightarrow {\mathbb{R}}^n$ are two curves on the same space.  Then the  Fr\'{e}chet distance between these two curves are: 
\begin{equation}
	d(P, Q) = \inf_{\alpha, \beta} \max_{t \in [0,1]} ||P(\alpha(t)) - Q(\beta(t))||
	\label{e:contfrechet}
\end{equation}
among all possible $\alpha: [0,1] \rightarrow [0,1]$ and $\beta: [0,1] \rightarrow [0,1]$, which are continuous reparameterisations of $P$ and $Q$ (respectively) that are non-decreasing and surjective. 

For computational efficiency, in this paper, we use the approximation of \eref{e:contfrechet} via discrete Fr\'{e}chet distance, approximating $P$ and $Q$ as polygonal chains, resulting in the following definition. Suppose $P': (p_1, p_2, p_3, \cdots, p_m)$ and $Q': (q_1, q_2, q_3, \cdots, q_{m'})$, where $p_i, q_j \in {\mathbb{R}}^n$ for $i \in [1, m], j \in [1, m']$. Then the  Fr\'{e}chet distance between these two polygonal chains are: 
\begin{eqnarray}
	&& d(P', Q') = \min_{k, l} \max_{t} ||p_{k(t)} - q_{l(t)}|| \nonumber\\
	&& \textrm{where:} \nonumber\\
	&& (i) \; k(t+1) = k(t)+1 \textrm{ and } l(t+1) = l(t) \textrm{, or} \nonumber\\
	&& (ii) \; k(t+1) = k(t) \textrm{ and } l(t+1) = l(t)+1
	\label{e:frechet}
\end{eqnarray}
Suppose $m \geq m'$, then $t \in [1, m]$, while $k(t)$ and $l(t)$ maps the index $t$ to an index of the points in $P'$ and $Q'$, respectively. This distance can be computed in $O(\frac{mm' \log\log m}{\log m})$ time and  $O(m+m')$ space\cite{agarwal2014computing}.


\subsubsection{POMDP}

To generate kamikaze trajectories, we need to consider the policy of the car being assessed. However, depending on the car's software system, this policy is not always explicit and may not even be known exactly prior to execution. Therefore, to the kamikaze trajectory generator, the car's policy is partially observed. Hence, we apply the POMDP to generate kamikaze trajectories. 

Formally, a POMDP is a tuple \pomdpTuple\cite{kaelbling1998planning}. At each time-step, a POMDP agent is in a state $\st \in \stSpace$, executes an action $\act \in \actSpace$, perceives an observation $\obs \in \obsSpace$, and moves to the next state $\stp \in \stSpace$.  The next state is distributed according to $\transF(\st, \act,\stp)$, which is a conditional probability function $P(\stp | \st, \act)$ that represents non-deterministic actions. The observation $\obs$ perceived depends on the observation function \obsF, which is a conditional probability function $P(\obs|\st, \act)$ that represents uncertainty in sensing. The notation \rewF is a reward function, from which the objective function is derived. The notations $\gamma \in (0, 1)$ is the discount factor to ensure that an infinite horizon POMDP problem remains a well-defined optimisation problem. 

The solution to a POMDP problem is an optimal policy \polOpt, that maps beliefs to actions in order to maximise the expected total reward, i.e. $V^*(\bel) = \max_{\act \in \actSpace} \left [ R(\bel, \act) + \gamma \sum_{\obs \in \obsSpace} Pr(\obs | \act, \bel) V^*(\tau(\bel, \act, \obs)) \right ]$, where $R(\bel, \act) = \sum_{\st \in \stSpace} \rewF(\st, \act)\bel(\st)$ and 
$Pr(\obs| \act, \bel) = \sum_{\stp \in \stSpace}\obsF(\stp, \act, \obs)\sum_{\st\in\stSpace}\transF(\st, \act, \stp)\bel(\st)$. The function $\tau$ computes the updated belief after the agent executes \act from \bel and perceives \obs.

\section{Overview of the Assessment Mechanism}

\begin{figure*}[!th]
    \centering
    \includegraphics[width=12cm]{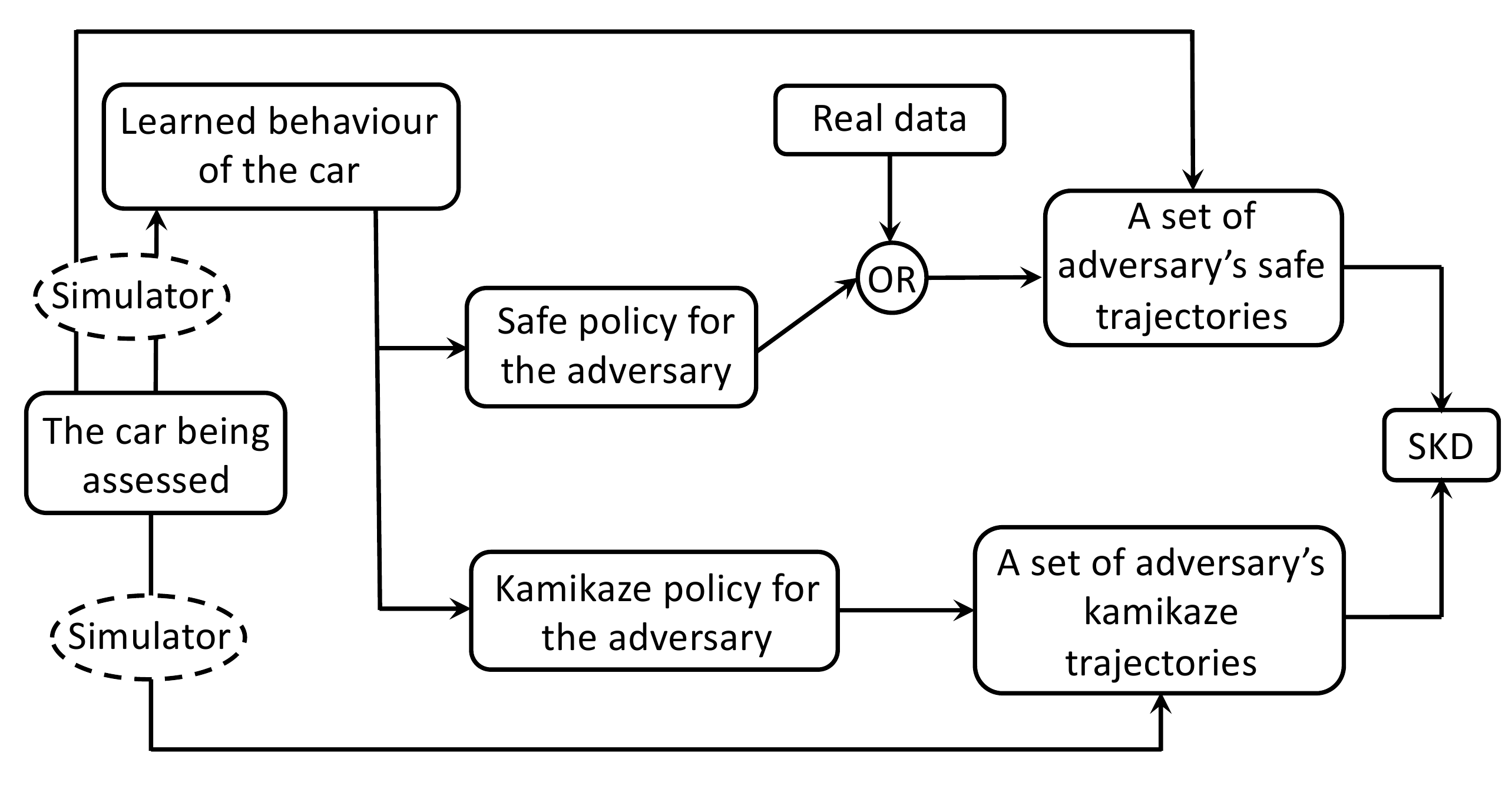}
    \caption{Proposed mechanism to assess the safety of an autonomous car. Dashed ellipse means it may or may not be used.}
    \label{f:systDiag}
\end{figure*}
\fref{f:systDiag} illustrates our proposed assessment mechanism. This mechanism automatically evaluates the safety of an autonomous car's software system as a holistic system. It accepts the software system of the car being assessed and uses a high-fidelity simulator or physical experiments to measure the car's safety under a testing scenario. Testing on the physical system directly is also applicable, for instance using a pedestrian dummy or robot. However, we believe such a test should be supplemented with simulation tests because most probably, we will only be able to perform a small number of physical trials and they will likely incur substantial cost.

A variety of testing scenarios can be used. The only requirement is that they must involve interactions between the car being assessed and an adversary (e.g., a pedestrian or another car), such that the adversary can  crash into the assessed car. Specifically, suppose $\adv = \sceneTuple$ defines an adversary, where: 
\begin{itemize}
	\item \env is a bounded environment, where the adversary and assessed car operates. Positions of objects in \env is defined in a bounded world coordinate space $\worldFrame \subset \mathbb{R}^2$. \\
	Note that our mechanism is sufficiently general for $\worldFrame$ to be in $\mathbb{R}^n$, and hence conceptually, the mechanism can be applied to a variety of robotics systems, such as Unmanned Aerial Vehicles, Autonomous Underwater Vehicles, and even manipulators, though further work on efficient computation of the components are required. 
	\item \oStSp is the set of states of the adversary, which includes the adversary's position in the world frame \worldFrame. The adversary's position is specified as the position of the 2D projection of the adversary's centre of mass in \worldFrame.
	\item \oActSp is the set of actions the adversary can perform. 
	\item \oDyn is a stochastic model of the adversary's dynamics function, which outputs a possible next state for the adversary after the action $\oAct \in \oActSp$ is performed from state $\oSt \in \oStSp$ at time $t'$ for time duration $\Delta_{t'}$ and perturbed by an error distribution \oDynUnc. 
\end{itemize} 
And, suppose the assessed car is defined as $\aCar = \aCarTuple$ where:
\begin{itemize}
	\item \aCarStSp is the set of states of the car being assessed, which includes the car's position in the world frame \worldFrame. The car's position is specified as the position of the 2D projection of the car's centre of mass in \worldFrame.
	\item \aCarActSp is the set of actions the assessed car can perform.
	\item \aCarDyn is the assessed car's dynamics function, which outputs a possible next state of the assessed car after the action $\aCarAct \in \aCarActSp$ is performed from state $\aCarSt \in \aCarStSp$ at time $t$ for a duration of $\Delta_t$ and perturbed by an error distribution \aCarDynUnc. This function may be a simplified function that is far from perfect. 
\end{itemize}
Then, the testing scenario must include the environment, states, and actions 
such that \aCarDyn may collide with \oDyn with substantial probability.

Any testing scenario that satisfies the above requirement can be used. To make the scenario concrete,  throughout this paper, we use the scenario where a pedestrian is crossing a single-lane road in front of the assessed car, which is similar to a scenario for testing autonomous emergency braking systems in Euro-NCAP\cite{EuroNcap}.


Given a testing scenario and an assessed car, our mechanism learns a predictive behaviour model of the assessed car under the scenario. The learning itself can be separately or in conjunction with kamikaze and safe trajectories generation of the adversary, via a high-fidelity simulator such as Carla\cite{CARLA17} or direct interaction with the assessed car. Many work\cite{Bha20:Online, Fri20:Efficient, Hoe17:Probabilistic, Sch18:Multiple} have been proposed to learn a predictive model of a self-driving car, and can be used with our mechanism.  

This learned model is used in generating kamikaze trajectories of the adversary. When real data of the adversary's safe  trajectories are not available, the model is also used to generate the safe trajectories. Details of these trajectory generations are provided in \sref{s:trajGen}.

Once the sets of safe and kamikaze trajectories are generated, our assessment mechanism computes the average Fr\'{e}chet distance between safe and kamikaze trajectories. This average distance can be used as a safety indicator, on how likely a safe trajectory becomes a dangerous one under a small deformation. Details on this distance computation and how it relates to the safety measure of the car being assessed is discussed in the next section.

\section{\noaLong (\noa)}

Suppose \safeTrajSet is the set of safe adversary's trajectories, specified in the world frame \worldFrame. Let \dieTrajSet{\safeTraj} be the set of kamikaze trajectories generated to be as close as possible to the safe trajectory $\safeTraj \in \safeTrajSet$, where $\dieTraj \in \dieTrajSet{\safeTraj}$ is specified in the world frame \worldFrame, and $\allDieTrajSet = \bigcup_{\safeTraj \in \safeTrajSet} \dieTrajSet{\safeTraj}$. Ideally, each \dieTrajSet{\safeTraj} contains kamikaze trajectories with exactly the same distance to \safeTraj, as we take the closest kamikaze trajectories. However, in practice, we approximate by sampling kamikaze trajectories within a pre-defined maximum distance from the safe trajectory. Details on this generation is in \sref{s:kamiTrajGen}. We then compute the \noa between these two sets of trajectories as:
\begin{equation}
	\noaF{\safeTrajSet, \allDieTrajSet} = \frac{1}{|\allDieTrajSet|} \sum_{\safeTraj \in \safeTrajSet} \sum_{\dieTraj \in \dieTrajSet{\safeTraj}} d(\safeTraj, \dieTraj)
	\label{e:trajDist} 
\end{equation}
where $d(\safeTraj, \dieTraj)$ is the discrete Fr\'{e}chet distance (\eref{e:frechet}) between the two trajectories. To ensure this more efficient Fr\'{echet} distance computation can be applied, we assume that trajectories in \safeTrajSet and \allDieTrajSet has been discretized uniformly in the time domain, which means the time to move between two consecutive points in a trajectory are the same everywhere. 

Notice that based on the definition of Fr\'{e}chet distance, $d(\safeTraj, \dieTraj) = \ds{\safeTraj, \dieTraj}$ implies that $\forall_{p \in \safeTraj} \exists_{q \in \dieTraj} \; q \in \ball{p, \ds{\safeTraj, \dieTraj}}$, where $\ball{p, \ds{\safeTraj, \dieTraj}} \subset \mathbb{R}^2$ is a ball centred at $p$ with radius $\ds{\safeTraj, \dieTraj}$. Therefore, $\noaF{\safeTrajSet, \allDieTrajSet}  = \noaD$ implies that on average, a safe adversary trajectory, sampled from the same distribution used to sample \safeTrajSet, can change into an unsafe trajectory after being deformed for less than or equal to \noaD distance away. Moreover, the probability of such a change happening after a very small deformation can be upper bounded by a function of \noaD. Specifically,
\begin{theorem}
	Suppose $\noaF{\safeTrajSet, \allDieTrajSet}  = \noaD$ and \dsRV is the random variable representing the Fr\'{e}chet distance $d(\safeTraj, \dieTraj)$ where $\safeTraj \in \safeTrajSet$ and $\dieTraj \in \dieTrajSet{\safeTrajSet} \in \allDieTrajSet$. Then for a small  real number $\eta \in (0, \delta)$,  $P(d(\safeTrajP, \dieTrajP) \leq \eta) \leq \frac{\var{\dsRV}+ 2 \eta \noaD}{\var{\dsRV} + {\noaD}^2}$, where \safeTrajP and \dieTrajP are any safe and kamikaze trajectories, sampled from the same distribution used to  generate \safeTrajSet and \allDieTrajSet, respectively.\label{th:probInd}
\end{theorem}
\begin{proof}
	Since all trajectories in $\safeTrajSet$ and $\allDieTrajSet$ are specified in a bounded space \worldFrame, \dsRV is a bounded random variable, and therefore based on Popoviciu inequality, \var{\dsRV} is finite. Furthermore, since \dsRV represents distance, $\dsRV \geq 0$. These two properties allow us to apply the Paley-Zygmund inequality on \dsRV to obtain:
	\begin{eqnarray}
			P(d(\safeTrajP, \dieTrajP) < \eta) &=& 1 - P\left(d(\safeTrajP, \dieTrajP) \geq \frac{\eta}{\noaD}\noaD\right) \nonumber\\
&\leq& 1 - \frac{(1-\frac{\eta}{\noaD})^2  {\noaD}^2}{\var{\dsRV} + {\noaD}^2}	\nonumber \\
&=& 1 - \frac{(\noaD-\eta)^2}{\var{\dsRV} + {\noaD}^2} \nonumber \\
&=& \frac{\var{\dsRV} + 2 \eta \noaD - {\eta}^2}{\var{\dsRV} + {\noaD}^2} \nonumber \\
		&\leq& \frac{\var{\dsRV} + 2 \eta \noaD}{\var{\dsRV} + {\noaD}^2} \nonumber 
	\end{eqnarray}	
giving the desired upper bound. 
\end{proof}

Although \var{\dsRV} appears in the above bound, we only propose \noa as a safety indicator. The reason is two folds. First, estimating variance is harder than estimating expected value. 

The second reason is since \dsRV is bounded and $E[\dsRV]$ is finite, \var{\dsRV} can be made sufficiently small by increasing the size of \safeTrajSet and $\dieTrajSet{\safeTraj}$ for each $\safeTraj \in \safeTrajSet$, thereby allowing the probability $P(d(\safeTrajP, \dieTrajP) \leq \eta)$ to be bounded from above by a function that depends only on $\eta$ and \noaD.  For instance, setting the above set of trajectories such that $\var{\dsRV} \leq 2 \eta$ will further bound the probability $P(d(\safeTrajP, \dieTrajP) \leq \eta) \leq 2 \left(\frac{\eta}{{\noaD}^2} + \frac{\eta}{\noaD} \right)$. 

Now, although one can reduce \var{\dsRV} to be arbitrarily small, the number of samples to ensure \var{\dsRV} is sufficiently small varies between one car and another. The reason is \var{\dsRV} is affected by uncertainty of the assessed system, including the variance on reaction time of the human driver if the car is semi-autonomous, uncertainty due to different road conditions, etc. too. 

Interestingly, the car's uncertainty is accounted in \noa too: Assuming all other conditions are the same, a car with larger stochastic uncertainty in its dynamics and perception will generate smaller \noa. This result may seem counter-intuitive, especially when the stochastic uncertainty over the effects of actions and observation are symmetric, considering \noa is an expected value. However, since our mechanism only computes the distance between a safe trajectory and kamikaze trajectories that are as close as possible to the safe trajectory, the symmetric effect is filtered out. Therefore, \noa will decrease as uncertainty increases, even if stochastic uncertainty in the effects of actions and observations are symmetric around its mean.

Intuitively, the above results show that \noa can be used to indicate the safety of an autonomous car. A car with  large \noa will have a small upper bound on the probability that an $\eta$ small deformation can turn a sampled safe trajectory of an adversary into an unsafe one.

\section{Generating Adversary's Trajectories}
\label{s:trajGen}

Given information about the environment and adversary $\adv = \sceneTuple$ and the assessed $\aCar = \aCarTuple$, to compute \noa, our mechanism requires sets of safe and kamikaze trajectories of the adversary. To generate each set of  trajectories, we construct a decision-making agent of the adversary, assuming that it has full observability about its own state and  partial observability about the assessed car's policy.  To account for partial observability of the car, the decision-making agent that represents the adversary is framed as a POMDP agent. 

If real data on safe trajectories or static trajectories from a regulatory body are available, the safe trajectories can use these trajectories too. However, if they are not available, the safe trajectories can be generated using POMDP.  

The following subsections describe the details of these POMDP models.


\subsection{Generating Kamikaze Trajectories}
\label{s:kamiTrajGen}

Given a collision-free trajectory of the adversary $\safeTraj \in \safeTrajSet$, the kamikaze agent generates a strategy that collides itself with the assessed car (under the test scenario used) as fast as possible while minimizing the total distance between its trajectory and \safeTraj. The agent has full observability about itself and potentially imperfect information about the car's dynamics, based on \aCarDyn in \aCar. However, it only has partial observability about the policy of the assessed car. Therefore, this kamikaze trajectory generation is a form of pursuit evasion under partial observability problem, and our mechanism uses POMDP to generate the policy. A kamikaze trajectory is then the traces of  the adversary's positions in a single simulation run of the POMDP policy.

To compute a kamikaze strategy, the kamikaze agent maintains a simplified predictive model of the assessed car $\simCar = \simCarTuple$, where $\simCarStSp \subset \aCarStSp$ and $\aCarPos \subseteq \simCarStSp$, and $\appCarDyn = \aCarDyn + {\mathcal{P}}'_{car}$, where $\aCarSt \in \simCarStSp$, $\aCarAct \in \aCarActSp$, and ${\mathcal{P}}'_{car}$ is a probability distribution function representing  the fact that the car dynamic $F_{car}$ provided in \aCar can be far from perfect. This distribution function is learned from data, which can be from simulation provided by the car's developers, or via data from running the physical car. In this paper, they are learned using a simple maximum likelihood method on data gathered from running a simulated car in the high fidelity simulator, Carla\cite{CARLA17}. 

Suppose the POMDP $\pomdpDie = \pomdpDieTuple$ represents the kamikaze agent. 
The state space $\stSpaceDie = \stSpaceDieAdv \times \simCarStSp \times \simCarPol$, where $\stSpaceDieAdv \subseteq \oStSp$,  $\oPos \subseteq \stSpaceDieAdv$. The kamikaze POMDP agent models a simplified policy of the assessed car as a parametric function, and \simCarPol represents the parameters of these policies. Note that by representing these parameters as a state variable, the uncertainty of the policy will be represented in the belief of the kamikaze agent. 

The action space $\actSpaceDie = \oActSp$. The transition function represents uncertainty in the resulting state of the adversary and assessed car. The transition function for the adversary and assessed car is conditionally independent given the current state. The adversary's transition function is defined as $\oDyn$, while the car's dynamic follows \appCarDyn.

The observation spaces and functions of the adversary depend on the type of sensors provided. However, the kamikaze agent assumes that the car model has full observability. 

The reward function \rewFDie is designed to encourage collision with the assessed car to happen while minimising distance of the adversary's position from the safe trajectory. To this end, we set the reward function, such that high reward is given when a collision between the adversary and assessed car happen, and higher penalty is given proportional to the distance between the adversary's position and \safeTrajSet in \worldFrame.

\subsection{Generating Safe Trajectories}
\label{s:safeTrajGen}

The POMDP agent $\pomdpSafe = \pomdpSafeTuple$ used to generate safe policies for the adversary is very similar to that used to generate the kamikaze policies. The only different is in \rewFSafe, the reward function is used to encourage the adversary to reach its intended destination as fast as possible without collision with the car. This optimality assumption follows the hypothesis in biology that human and other living beings generally tend to optimise their objective functions\cite{Breed:2015}, though in general the objective functions being optimised are unclear. In our test mechanism, we assume the objective function follows the objective function of the POMDP \pomdpSafe.

Note that the safe trajectories do not need to use the same car model as the one being assessed nor the one used to generate the kamikaze trajectories. After all, these synthetically generated safe trajectories are supposed to replace real world trajectories data of the adversary, which may not be operating against the assessed car. 

\section{Experiments}

The aim of our experiments is two folds. First is to understand how reasonable our \noa as an indicator of car safety. Second is to understand the required time for our proposed mechanism to output this safety indicator. 

\subsection{Scenarios}

To achieve our experimental goals, we use the high level scenarios of avoiding collision with a pedestrian crossing  the street, when the car is moving forward in a single lane. This scenario is similar to one of the scenarios used to test Autonomous Emergency Breaking systems in Euro-NCAP\cite{EuroNcap}.


\subsubsection{The Cars}
\label{s:expACar}

We assume the car is moving forward in a single-lane road when a pedestrian is crossing the road. We test our proposed testing mechanism on two types of car controllers. 

The first type is the set of \emph{basic controllers}. This type of controllers is designed to systematically test \noa. These controllers assume the car starts from a given maximum velocity $8\frac{1}{3} m/s$ and apply the following policy:
\begin{equation*}
	\hspace{-6cm} \simCarPol(\aCarPos, \oPos) = 
\end{equation*}
\begin{equation*}
	\left\{ 
	\begin{array}{ll}
		\acc = -3.5m/s^2 + {\mathcal{U}}[-0.1\acc, 0.1\acc] & \nonumber \\
		\hspace{2cm}\textrm{longitudinal distance(\aCarPos, \oPos)} \leq \mult \times \safeDist \nonumber \\
		\acc = 0 & \hspace{-2cm}\textrm{Otherwise} \nonumber
	\end{array} \right.
	\nonumber
\end{equation*}
where \acc is the acceleration applied for a duration of $0.3s$ and ${\mathcal{U}}[-0.1\acc, 0.1\acc]$ is uniform distribution with support $[-0.1\acc, 0.1\acc]$, representing the car's uncertainty in the exact deceleration it performs. In addition to this uncertainty, the car's velocity is influenced by uncertainty too, such that the evolution of its speed is governed by $\vel_{t+1} = \vel_{t} + {\mathcal{U}}[-0.05 \vel_{t}, \, 0.05 \vel_{t}] + 0.3 \acc $. The notation \safeDist is  a safe distance threshold, defined as the distance for the car to move from the maximum velocity \maxVel m/s to $0$ m/s, assuming maximum deceleration (in this case, $3.5 m/s^2$) is applied and the car's motion is deterministic. This type of controller can be made more or less aggressive by applying a different multiplier \mult. Higher \mult means the controller has more margin of error, and therefore is safer. For our experiments, we apply 9 different values for \mult: $\{0.5, 0.625, 0.75, 0.875, 1.0, 1.05, 1.1, 1.125, 1.15\}$. The lower $C$ is, the more aggressive its behaviour is. Simulation run for this basic controllers are run on a simple C++ implementation based on the given model.

The second type of the car controller is the Machine-Learning based controller (ML) \cite{MLController2018}, which became a default controller of Carla. For our purpose, we ran the same controller, but focus only on two of the longitudinal control states \emph{Cruising} and \emph{Hazard Stop}. The other longitudinal control states [following, red  light, and over limit] were not relevant for the scenarios in our experiments. This controller is ran on Carla v.0.9.6, that allows pedestrian control for adversary interactions. 

\subsubsection{The Adversary}

The adversary in this scenario is a pedestrian crossing a single-lane road in front of the car. We assume when the pedestrian moves, the pedestrian is moving with a constant velocity of $2.5m/s$. We also assume that the pedestrian motion is deterministic. It perceives observation on the distance between itself and the assessed car, but this observation is noisy.

\subsection{Trajectories Generators}

\subsubsection{Safe Trajectories}

We use two sets of safe trajectories for the pedestrian. The first set of safe trajectories is a real world data set, which is extracted from the scenario \emph{Lateral Interaction (Unilateral)} published in \cite{DUTRealData2019}. 

The second set of safe trajectories is synthetically generated using POMDP as described in \sref{s:safeTrajGen}, using a high fidelity simulator Carla v.0.9.6\cite{CARLA17} when the car is controlled using the Machine Learning controller\cite{MLController2018}. To generate this trajectory, the pedestrian actions are deterministic, with the action space being \(A^{\circ} = \langle North,South, East,NorthEast,SouthEast,Stay\rangle\), where the first five actions in the set correspond to movement actions of the agent along the cardinal direction with a displacement magnitude of \(|d| = 2.5 m/s\), and the last action allows the adversary to remain in its position. The car model used by the POMDP agent is based on maximum likelihood, learned using simulated data from Carla v.0.9.6.  The adversary observes the distance between itself and the car, with noise given by Gaussian distribution  ${\mathcal{N}}(0, 1)$.

\subsubsection{Kamikaze Trajectories}

The kamikaze trajectories are generated using POMDP as described in \sref{s:kamiTrajGen}. 

The pedestrian model of this POMDP is the same as that used by the POMDP to generate safe trajectories. However, the model of the assessed car needs to be in-line with \aCar. 

For the basic controller, the POMDP agent knows the controller's multiplier \mult, but not the full uncertainty plaguing the car controller (described in \sref{s:expACar}. Specifically, it only models the velocity uncertainty and not the acceleration uncertainty in its car model. Furthermore, it has a noisy estimate of its relative position to the assessed car. 

For the ML controller, we use the same learned car model used by the POMDP to generate safe trajectories. 

Last but not least, when real data is used as the safe trajectory, we adjusted the car model to the data. We adjusted the car's geometry to match the car in the real data. The adjusted size is smaller than the car model used in the basic controller. Therefore, we also adjusted the parameters of its dynamics to ensure that collision may still happen. 

\subsection{Setup}

All trajectory generations, simulation runs, and \noa computation were ran on a desktop with an Intel Core i7-8700 @ 3.20 GHz CPU and 32GB RAM and NVidia GeForce GTX1060 with 6GB dedicated RAM. The GPU is used only for Carla simulation. To generate POMDP-based safe and kamikaze trajectories, we use OPPT\cite{hoerger18:OPPT}, which is a POMDP toolkit for on-line POMDP solving, designed to ease applying POMDPs to robot planning problems. To compute \noa, we use the algorithm and implementation as described in\cite{DiscreteFretchet1994}.

\subsection{Results}

To assess \noa systematically, we applied our testing mechanism to assess the basic car controller with varying multiplier \mult, as described in \sref{s:expACar}. For this assessment, for each \mult value used and each type of safe trajectories (real and synthetic data), we generated 5 safe trajectories and 100 kamikaze trajectories per safe trajectory. We then computed \noa for each basic controller and each type of safe trajectories by averaging the Fr\'{e}chet distance of the 500 pairs of safe and kamikaze trajectories. The \noa and its 95\% confidence interval, along with its collision rate are presented in \fref{f:skdCol}(a).

\begin{figure}[!ht]
	\begin{tabular}{cc}
		\includegraphics[width=8.5cm]{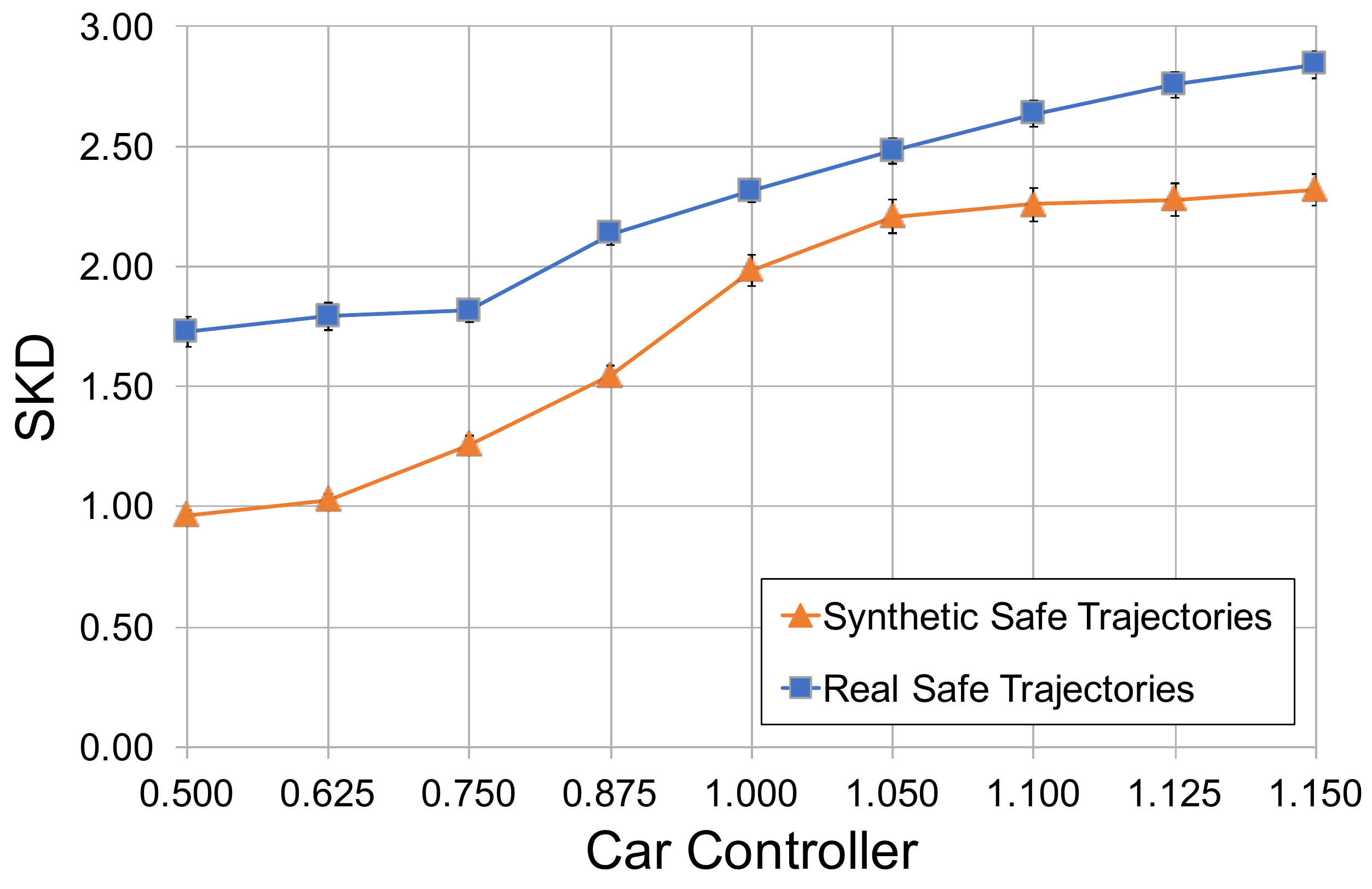} & \includegraphics[width=8.5cm]{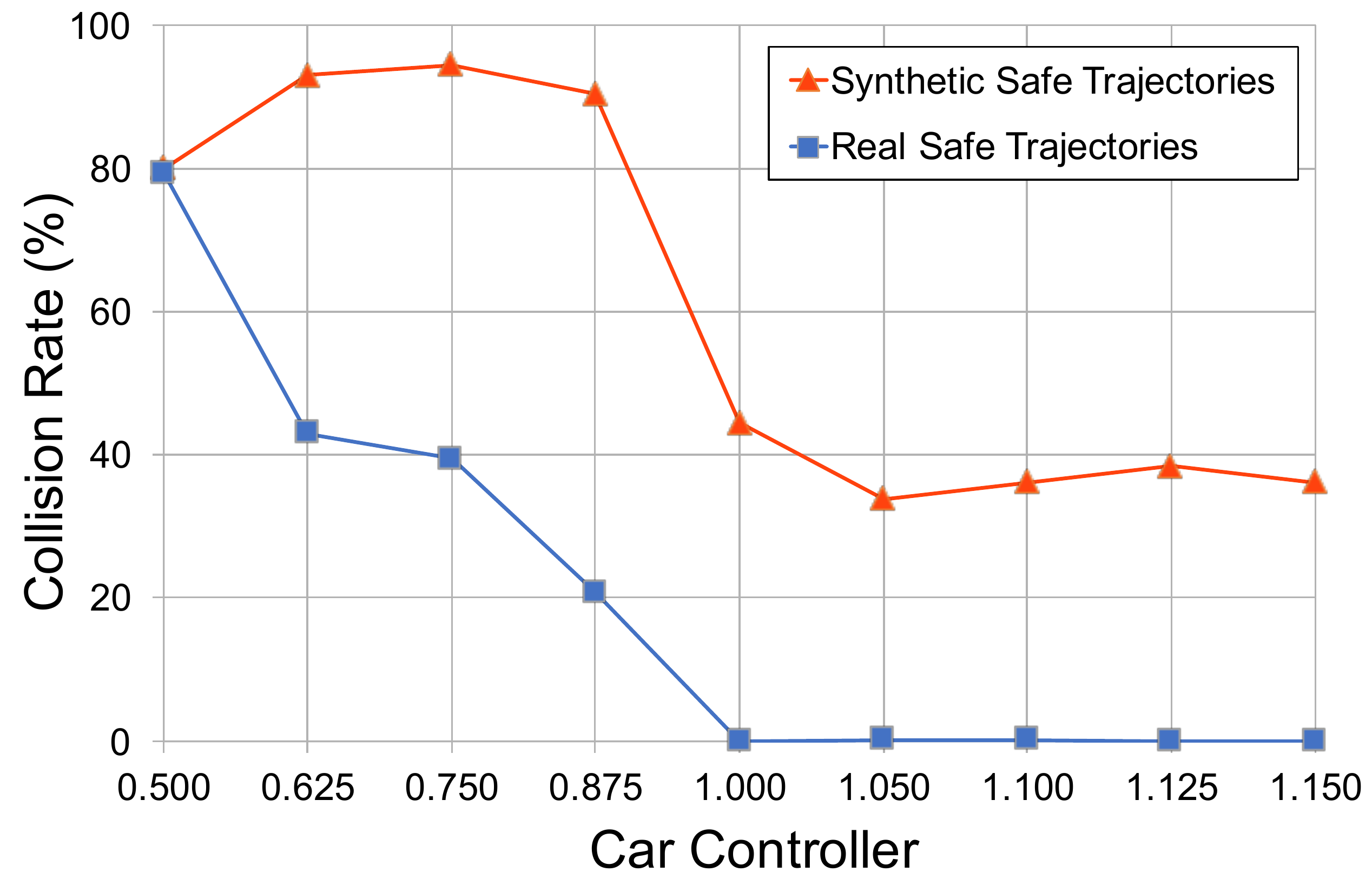} \\
		(a) & (b) \\
	\end{tabular}
	\caption{\noa and collision rate results.}
	\label{f:skdCol}
\end{figure}

They indicate that \noa increases as \mult increases, regardless of how the safe trajectories are generated, which is in-line with our design that as \mult increases, the controller has more margin of error, and hence are safer.  

For verification purposes on how safe or dangerous the different controllers are, for each \mult value used and each safe trajectories, we also compute the collision rate between 500 simulations of the car \aCar moving forward, while the pedestrian follows the safe trajectories  to cross the road. The results are presented in \fref{f:skdCol}(b).

It is interesting to compare \noa across the two types of safe trajectories. The consistently higher \noa in the safe trajectories extracted from real data is mostly in-line with the consistently lower collision rate of this set of scenarios. 

This above trend only differs when $\mult = 0.5$. The collision rate for this particular controller when the safe trajectories are extracted from real data is only slightly lower than when the safe trajectories are synthetically generated (79\% vs 80\%), while their \noa differs by more than $0.5$ points. This controller is the only one where on average, the car only starts to stop when it is already too late (i.e., when the car can no longer avoid collision). In such a situation, error in the acceleration of the car, which dominates error in the velocity (\sref{s:expACar}), can substantially influence the collision rate.  Since this acceleration error is not modelled in the POMDP that generates the kamikaze trajectories, this dominating factor is not considered in \noa computation. A possible remedy and somewhat typical in practical applications of POMDP is to model the system with larger uncertainty. 

\begin{table}[h]
\centering
\caption{ML-Controller Assessment}
\label{t:mlRes}
\begin{tabular}{|c|c|c|}
\hline
\multicolumn{3}{|c|}{\textbf{ML-Controller}}                        \\ \hline
\textbf{Safe Trajectories} & \textbf{\begin{tabular}[c]{@{}c@{}}SKD \(\pm\) 95 C.I\end{tabular}} & \textbf{Collision Rate (\%)} \\ \hline
Synthetic Data   & 2.31 \(\pm\) 0.08   & 0   \\ \hline
Real Data        & 5.70 \(\pm\) 0.18   & 0   \\ \hline
\end{tabular}
\end{table}
We also tested our assessment mechanism on the ML controller. For this purpose, we used the same 5 safe trajectories per type as above and 100 kamikaze trajectories per safe trajectory, and simulate them in Carla v.0.9.6. We computed \noa and collision rate based on these 500 pairs of safe and kamikaze trajectories. The results are in \tref{t:mlRes}.

\begin{sidewaystable}[!h]
\centering
\caption{Time taken to generate safe and kamikaze trajectories, Fr\'{e}chet distance, \noa, and the total time.}
\label{t:timeRes}
\begin{tabular}{|c|c|c|c|c|c|c|c|}
\hline
\multirow{2}{*}{\textbf{\begin{tabular}[c]{@{}c@{}}Car\\ Controller\end{tabular}}} &
  \multicolumn{4}{c|}{\textbf{Synthetic Data}} &
  \multicolumn{3}{c|}{\textbf{Real Data}} \\ \cline{2-8} 
 &
  \textbf{\begin{tabular}[c]{@{}c@{}}SafeTrajGen\\ Avg \\ 95 C.I (ms)\end{tabular}} &
  \textbf{\begin{tabular}[c]{@{}c@{}}KamikazeGen\\ Avg \\ 95 C.I (ms)\end{tabular}} &
  \textbf{\begin{tabular}[c]{@{}c@{}}FretchetTime\\ Avg \\ 95 C.I (ms)\end{tabular}} &
  \textbf{\begin{tabular}[c]{@{}c@{}}Total SKD\\ Avg \\ 95 C.I (ms)\end{tabular}} &
  \textbf{\begin{tabular}[c]{@{}c@{}}KamikazeGen\\ Avg \\ 95 C.I (ms)\end{tabular}} &
  \textbf{\begin{tabular}[c]{@{}c@{}}FretchetTime\\ Avg \\ 95 C.I (ms)\end{tabular}} &
  \textbf{\begin{tabular}[c]{@{}c@{}}Total SKD\\ Avg \\ 95 C.I (ms)\end{tabular}} \\ \hline
0.500 &
  5535.41 \(\pm\) 3.10 &
  2437.99 \(\pm\) 3.10 &
  0.60 \(\pm\) 0.01 &
  7974 \(\pm\) 127.19 &
  1950.79 \(\pm\) 21.98 &
  0.398 \(\pm\) 0.008 &
  1951.19 \(\pm\) 21.98 \\ \hline
0.625 &
  5535.41 \(\pm\) 3.10 &
  2452.62 \(\pm\) 3.00 &
  0.62 \(\pm\) 0.01 &
  7988.65 \(\pm\) 127.19 &
  1975.86 \(\pm\) 22.17 &
  0.404 \(\pm\) 0.008 &
  1976.26 \(\pm\) 22.17 \\ \hline
0.750 &
  5535.41 \(\pm\) 3.11 &
  2471.53 \(\pm\) 0.60 &
  0.65 \(\pm\) 0.01 &
  8007.59 \(\pm\) 127.19 &
  2133.56 \(\pm\) 17.48 &
  0.455 \(\pm\) 0.011 &
  2134.02 \(\pm\) 17.48 \\ \hline
0.875 &
  5535.41 \(\pm\) 3.11 &
  2741.11 \(\pm\) 10.46 &
  0.72 \(\pm\) 0.02 &
  8277.24 \(\pm\) 127.59 &
  2173.64 \(\pm\) 18.15 &
  0.414 \(\pm\) 0.008 &
  2174.05 \(\pm\) 18.15 \\ \hline
1.000 &
  5535.41 \(\pm\) 3.12 &
  3023.28 \(\pm\) 29.95 &
  0.76 \(\pm\) 0.02 &
  8559.45 \(\pm\) 130.64 &
  2297.44 \(\pm\) 21.20 &
  0.455 \(\pm\) 0.010 &
  2297.89 \(\pm\) 21.2 \\ \hline
1.050 &
  5535.41 \(\pm\) 3.12 &
  4172.54 \(\pm\) 145.18 &
  0.78 \(\pm\) 0.02 &
  9708.72 \(\pm\) 192.99 &
  2364.42 \(\pm\) 25.75 &
  0.465 \(\pm\)0.011 &
  2364.89 \(\pm\) 25.75 \\ \hline
1.100 &
  5535.41 \(\pm\) 3.13 &
  4038.40 \(\pm\) 124.23 &
  0.77 \(\pm\) 0.02 &
  9574.58 \(\pm\) 177.77 &
  2473.88 \(\pm\) 42.84 &
  0.490 \(\pm\) 0.013 &
  2474.37 \(\pm\) 42.84 \\ \hline
1.125 &
  5535.41 \(\pm\) 3.13 &
  3491.84 \(\pm\) 101.59 &
  0.82 \(\pm\) 0.02 &
  9028.07 \(\pm\) 162.76 &
  2380.11 \(\pm\) 20.21 &
  0.498 \(\pm\) 0.013 &
  2380.61 \(\pm\) 20.21 \\ \hline
1.150 &
  5535.41 \(\pm\) 3.14 &
  3511.95 \(\pm\) 43.86 &
  0.80 \(\pm\) 0.02 &
  9048.16 \(\pm\) 134.51 &
  2482.93 \(\pm\) 18.69 &
  0.487 \(\pm\) 0.012 &
  2483.41 \(\pm\) 18.69 \\ \hline
ML &
  5535.41 \(\pm\) 3.14 &
  5391.41 \(\pm\) 100.36 &
  0.83 \(\pm\) 0.02 &
  10927.65 \(\pm\) 161.99 &
  7985.79 \(\pm\) 165.88 &
  1.051 \(\pm\) 0.029 &
  7986.85 \(\pm\) 165.88 \\ \hline
\end{tabular}
\end{sidewaystable}
To understand the feasibility of conducting the proposed testing mechanism frequently, we need to look into the time required to compute the above results. \tref{t:timeRes} presents the time taken for safe and kamikaze trajectory generations, Fr\'{e}chet distance computation, and the total time. These results indicate that a simulation run, which is one data point to compute \noa, can be computed in under 11s. The time taken to assess the ML controller is longer than that to assess the basic controller because of the high-fidelity simulator used and the more complex computation required by the controller. However, even the ML controller took less than 11ms. Moreover, if we compute the synthetic safe trajectories ahead of time and reuse them ---a practice that will reduce the variance when computing \noa---, the time required for a single simulation run regardless of the controller is under 8s. Of course, to have statistical confidence, we need to run these simulations many times. However, this process is embarrassingly parallel. 

Of course, if the tests were to be conducted directly on the physical system, more time and effort would be required to run them multiple times. In this case, we imagine combining physical and simulation results 
would be beneficial to reduce the number of repeated physical robot testing required. But, further work are necessary to combine the two results well.

\section{Summary}

In this paper, we propose a testing mechanism to assess the safety of  autonomous cars, inspired by the NCAP safety rating. Core to our proposal is a similarity measure \noa, which uses the Fr\'{e}chet distance between the adversary's safe trajectories and kamikaze trajectories closest to those safe trajectories. We found the average of such a distance is inversely proportional to the upper bound of the probability that a small deformation turns a safe trajectory into a dangerous one. Systematic tests on simulation and a test on a Machine-Learning controller using a high-fidelity simulator corroborate this characteristics of \noa. 

The time taken for each simulation test is under 11 seconds if we include parts of the scenario generations, and under 8 seconds otherwise. Therefore, it is feasible for the proposed testing mechanism to generate sufficient statistics from simulation on a quad-core desktop in 15-30 minutes, which is equivalent to the typical time one uses for washing a car. The speed of this assessment opens the possibility for a relatively frequent safety assessment to be taken in simulation, for instance after every patch or software update is performed.

Moreover, the safe trajectory can be based on static scenarios as defined by regulatory bodies, which will likely be more acceptable by the regulators.

Future work abounds. More exhaustive testing using a variety of state-of-the-art Machine Learning controllers  are needed to better understand the effectiveness of \noa as a safety measure. In the computation of the component of the testing mechanism, plenty of rooms are available to improve the  kamikaze trajectory generation. In this paper, we assume a very simple car dynamics. Applying a more realistic car dynamics would require solving differential game under partial observability. Furthermore, ensuring the kamikaze trajectory generated is close to the corresponding safe trajectory requires solving constrained differential game problems. 
Currently, \noa is computed based on simulation results alone. How could we combine tests on simulation test and physical robot, so as to obtain a good safety measure fast? 

Last but not least, we believe planning under uncertainty techniques could help develop viable \emph{user-focused} safety indicator and testing mechanisms for autonomous systems, and hope this work encourages further exploration in this direction.

\section*{Acknowledgments}

This work is supported by the Assuring Autonomy International Programme and ANU Futures Scheme.


\bibliographystyle{plainnat}
\bibliography{references}

\end{document}

%% file: shortcut-forArxiv.tex
\newtheorem{theorem}{Theorem}
\newenvironment{proof}[1][Proof]{\begin{trivlist}\item[\hskip \labelsep {\bfseries #1}]}{\end{trivlist}}

\newcommand{\sref}{Section~\ref}

\newcommand{\eref}[1]{eq. (\ref{#1})}
\newcommand{\fref}[1]{Figure~\ref{#1}}

\newcommand{\tref}[1]{Table~\ref{#1}}

\newcounter{comment}
\newcommand{\noa}{SKD\xspace}
\newcommand{\noaLong}{Safe--Kamikaze Distance\xspace}

\newcommand{\noaF}[1]{\ensuremath{SKD(#1)}\xspace}
\newcommand{\noaD}{\ensuremath{\delta}\xspace}

\newcommand{\ds}[1]{\ensuremath{\delta_{#1}}\xspace}
\newcommand{\dsRV}{\ensuremath{D}\xspace}

\newcommand{\pomdpTuple}{\ensuremath{\langle S, A, T, O, Z, R, \gamma \rangle}\xspace}
\newcommand{\pomdpSafe}{\ensuremath{{\mathcal{P}}^{\circ}}\xspace}
\newcommand{\pomdpSafeTuple}{\ensuremath{\langle S^{\circ}, A^{\circ}, T^{\circ}, O^{\circ}, Z^{\circ}, R^{\circ}, \gamma^{\circ} \rangle}\xspace}
\newcommand{\pomdpDie}{\ensuremath{{\mathcal{P}}^{\dagger}}\xspace}
\newcommand{\pomdpDieTuple}{\ensuremath{\langle S^{\dagger}, A^{\dagger}, T^{\dagger}, O^{\dagger}, Z^{\dagger}, R^{\dagger}, \gamma^{\dagger} \rangle}\xspace}

\newcommand{\stSpace}{\ensuremath{S}\xspace}

\newcommand{\stSpaceDie}{\ensuremath{S^{\dagger}}\xspace}
\newcommand{\stSpaceDieAdv}{\ensuremath{S_{adv}^{\dagger}}\xspace}

\newcommand{\st}{\ensuremath{s}\xspace}
\newcommand{\stp}{\ensuremath{s'}\xspace}

\newcommand{\actSpace}{\ensuremath{A}\xspace}

\newcommand{\actSpaceDie}{\ensuremath{A^{\dagger}}\xspace}
\newcommand{\act}{\ensuremath{a}\xspace}

\newcommand{\obsSpace}{\ensuremath{O}\xspace}

\newcommand{\obs}{\ensuremath{o}\xspace}

\newcommand{\transF}{\ensuremath{T}\xspace}

\newcommand{\obsF}{\ensuremath{Z}\xspace}

\newcommand{\rewF}{\ensuremath{R}\xspace}
\newcommand{\rewFSafe}{\ensuremath{R^{\circ}}\xspace}
\newcommand{\rewFDie}{\ensuremath{R^{\dagger}}\xspace}

\newcommand{\bel}{\ensuremath{b}\xspace}

\newcommand{\polOpt}{\ensuremath{\pi^*}\xspace}

\newcommand{\simCarPol}{\ensuremath{\Pi_{car}^{\dagger}}\xspace}

\newcommand{\worldFrame}{\ensuremath{W}\xspace}
\newcommand{\sceneTuple}{\ensuremath{\langle E, S_{adv}, A_{adv}, F_{adv} \rangle}\xspace}

\newcommand{\adv}{\ensuremath{Adv}\xspace}
\newcommand{\oStSp}{\ensuremath{S_{adv}}\xspace}

\newcommand{\oSt}{\ensuremath{s_{adv}}\xspace}
\newcommand{\oPos}{\ensuremath{W_{adv}}\xspace}

\newcommand{\oActSp}{\ensuremath{A_{adv}}\xspace}

\newcommand{\oAct}{\ensuremath{a_{adv}}\xspace}
\newcommand{\oDyn}{\ensuremath{F_{adv}(s_{adv}, a_{adv}, t', \Delta_{t'}, {\mathcal{P}}_{adv})}\xspace}
\newcommand{\oDynUnc}{\ensuremath{{\mathcal{P}}_{adv}}\xspace}
\newcommand{\env}{\ensuremath{E}\xspace}

\newcommand{\aCar}{\ensuremath{Car}\xspace}
\newcommand{\aCarTuple}{\ensuremath{\langle S_{car}, A_{car}, F_{car}\rangle}\xspace}
\newcommand{\aCarStSp}{\ensuremath{S_{car}}\xspace}

\newcommand{\aCarSt}{\ensuremath{s_{car}}\xspace}
\newcommand{\aCarPos}{\ensuremath{W_{car}}\xspace}

\newcommand{\aCarActSp}{\ensuremath{A_{car}}\xspace}

\newcommand{\aCarAct}{\ensuremath{a_{car}}\xspace}
\newcommand{\aCarDyn}{\ensuremath{F_{car}(s_{car}, a_{car}, t, \Delta_{t}, {\mathcal{P}}_{car})}\xspace}
\newcommand{\aCarDynUnc}{\ensuremath{{\mathcal{P}}_{car}}\xspace}

\newcommand{\simCar}{\ensuremath{\widehat{Car}}\xspace}
\newcommand{\simCarTuple}{\ensuremath{\langle \widehat{S_{car}}, A_{car}, \widehat{F_{car}}\rangle}\xspace}
\newcommand{\simCarStSp}{\ensuremath{\widehat{S_{car}}}\xspace}

\newcommand{\appCarDyn}{\ensuremath{\widehat{F_{car}}(s_{car}, a_{car}, t, \Delta_{t}, {\mathcal{P}}'_{car})}\xspace}

\newcommand{\safeTraj}{\ensuremath{\phi}\xspace}

\newcommand{\safeTrajP}{\ensuremath{\phi'}\xspace}
\newcommand{\safeTrajSet}{\ensuremath{\Phi}\xspace}

\newcommand{\dieTraj}{\ensuremath{\psi}\xspace}
\newcommand{\dieTrajP}{\ensuremath{\psi'}\xspace}
\newcommand{\dieTrajSet}[1]{\ensuremath{\Psi(#1)}\xspace}
\newcommand{\allDieTrajSet}{\ensuremath{\Psi}\xspace}

\newcommand{\ball}[1]{\ensuremath{B(#1)}\xspace}

\newcommand{\var}[1]{\ensuremath{Var(#1)}\xspace}

\newcommand{\safeDist}{\ensuremath{\kappa}\xspace}
\newcommand{\mult}{\ensuremath{C}\xspace}
\newcommand{\maxVel}{\ensuremath{\nu_{M}}\xspace}
\newcommand{\acc}{\ensuremath{acc_{car}}\xspace}
\newcommand{\vel}{\ensuremath{\nu}\xspace}